\documentclass[a4paper,twoside]{article}

\usepackage{amsmath}
\usepackage{amsfonts}
\usepackage{amssymb}
\usepackage{amsbsy}
\usepackage{theorem}
\usepackage{algorithm}
\usepackage{algorithmic}
\usepackage{paralist}
\usepackage{psfrag}
\usepackage{pst-tree}
\usepackage{natbib}

\theoremstyle{plain} \newtheorem{remark}{Remark}
\theoremstyle{plain} \newtheorem{definition}{Definition}
\theoremstyle{plain} 
\theoremstyle{plain} 
\theoremstyle{plain} \newtheorem{proposition}{Proposition}
\theoremstyle{plain} \newtheorem{lemma}{Lemma}
\theoremstyle{plain} 
\theoremstyle{plain} \newtheorem{assumption}{Assumption}
\theoremstyle{plain} 
\theoremstyle{plain}

\newcommand{\qed}{\nobreak \ifvmode \relax \else
      \ifdim\lastskip<1.5em \hskip-\lastskip
      \hskip1.5em plus0em minus0.5em \fi \nobreak
      \vrule height0.5em width0.5em depth0.25em\fi}

\newenvironment{proof}[1][Proof]{\begin{trivlist}
\item[\hskip \labelsep {\bfseries #1}]}{\qed\end{trivlist}}

\newcommand\Reals {{\mathbb{R}}}

\newcommand\FB {{\mathfrak{B}}}
\newcommand\FM {{\mathfrak{M}}}

\newcommand\CA {{\mathcal{A}}}

\newcommand\CC {{\mathcal{C}}}

\newcommand\CL {{\mathcal{L}}}
\newcommand\CM {{\mathcal{M}}}

\newcommand\CO {{\mathcal{O}}}

\newcommand\CR {{\mathcal{R}}}
\newcommand\CS {{\mathcal{S}}}
\newcommand\CT {{\mathcal{T}}}

\newcommand\argmax{\mathop{\rm arg\,max}}

\newcommand \E {\mathop{\mbox{\ensuremath{\mathbf{E}}}}\nolimits}
\newcommand \Ep[1] {\mathop{\mbox{\ensuremath{\mathbf{E}_{#1}}}}}
\newcommand \Exi {\Ep{\xi}}

\renewcommand \Pr {\mathop{\mbox{\bf{P}}}\nolimits}

\newcommand \hx {\hat{x}}

\newcommand \pl{{\pi}}

\newcommand \inftynorm[1]{\left\|#1\right\|_\infty}

\newcommand \defn {\mathrel{\triangleq}}
\newcommand \MDPs {\ensuremath{\CM}}

\newcommand \lqed {}

\newcommand \st {s_t}
\newcommand \rt {r_t}

\newcommand \at {a_t}
\newcommand \xit {\xi_t}

\newcommand \wt {\omega_t}

\newcommand \stn {s_{t+1}}
\newcommand \rtn {r_{t+1}}

\newcommand \xitn {\xi_{t+1}}
\newcommand \wtn {\omega_{t+1}}

\newcommand \Vtpm {{V_{t,T}^{\pl,\mu}}}

\newcommand \meanMDP {{\bar{\mu}_\xi}}

\newcommand \VL {V_L}

\newcommand \VU {V_U}

\newcommand \Beliefs {\FB}

\newcommand \Vb {V^b}
\newcommand \Vbp {V^{b'}}

\newcommand \hV {\hat{V}}
\newcommand \hVL {\hat{V}_L}
\newcommand \hVU {\hat{V}_U}

\newcommand \tv {\tilde{v}}

\newcommand \hbs {\hat{b}^*}
\newcommand \hVLb {\hVL^{b}}

\newcommand \hVUb {\hVU^{b}}

\newcommand \Vhbs {V^{\hbs}}

\newcommand \VLb {\VL^{b}}
\newcommand \VLbp {\VL^{b'}}

\newcommand \VUb {\VU^{b}}

\newcommand \hDL {\hat{\Delta}_L}

\newcommand \Lceil {\left\lceil}
\newcommand \Rceil {\right\rceil}

\newcommand {\dd} {\,\mathrm{d}}

\newcommand\cset[2] {\left\{#1 \mathrel{:} #2\right\}}

\usepackage{epsfig}
\usepackage{subfigure}
\usepackage{calc}
\usepackage{amssymb}
\usepackage{amstext}
\usepackage{amsmath}
\usepackage{multicol}
\usepackage{pslatex}

\usepackage{caption}

\newcommand \SectStyle {}

\begin{document}

\author{Christos Dimitrakakis}

\title{Complexity of stochastic branch and bound methods for belief tree search 
in Bayesian reinforcement learning}

\maketitle

\begin{abstract}
  There has been a lot of recent work on Bayesian methods for
  reinforcement learning exhibiting near-optimal online performance.
  The main obstacle facing such methods is that in most problems of
  interest, the optimal solution involves planning in an infinitely
  large tree. However, it is possible to obtain stochastic lower and
  upper bounds on the value of each tree node. This enables us to use
  stochastic branch and bound algorithms to search the tree
  efficiently. This paper proposes two such algorithms and examines
  their complexity in this setting.
\end{abstract}


\section{\SectStyle{Introduction}}

Various Bayesian methods for exploration in Markov decision processes
(MDPs) and for solving known partially-observable Markov decision
processes (POMDPs), were proposed previously
(c.f.~\citep{poupart2006asd,duff2002olc,RossPineau:OnlinePlanningPOMDPs:jmlr2008}). However,
such methods often suffer from computational tractability problems.
Optimal Bayesian exploration requires the creation of an augmented MDP
model in the form of a tree~\citep{duff2002olc}, where the root node
is the current belief-state pair and children are all possible
subsequent belief-state pairs.  The size of the belief tree increases
exponentially with the horizon, while the branching factor is infinite
in the case of continuous observations or actions.

In this work, we examine the complexity of efficient algorithms for
expanding the tree. In particular, we propose and analyse stochastic
search methods similar to the ones proposed
in~\citep{DBLP:conf/nips/BubeckMSS08,Norkin:StochasticBnB:MP96}. Related
methods have been previously examined experimentally in the context of
Bayesian reinforcement learning
in~\citep{dimitrakakis:cimca08,wang:bayesian-sparse-sampling:icml:2005}.

The remainder of this section summarises the Bayesian planning
framework. Our main results are presented in
Sect.~\ref{sec:pac-bounds}. Section~\ref{sec:related_work} concludes
with a discussion of related work. Technical proofs and related
results are presented in the Appendix.

\subsection{Markov Decision Processes}
\label{sec:bamdps}
Reinforcement learning [c.f. \cite{Puterman:MDP:1994}] is discrete-time
sequential decision making problem, where we wish to act so as to
maximise the expected sum of discounted future rewards $\E
\sum_{k=1}^{T} \gamma^k r_{t+k}$, where $r_t \in \Reals$ is a
stochastic reward at time $t$. We are only interested in rewards from
time $t$ to $T > 0$, and $\gamma \in [0,1]$ plays the role of a
discount factor. Typically, we assume that $\gamma$ and $T$ are known
(or have known prior distribution) and that the sequence of rewards
arises from a Markov decision process $\mu$:
\begin{definition}[MDP]
  A Markov decision process is a discrete-time stochastic process
  with: A state $s_t \in \CS$ at time $t$ and a reward $r_t \in
  \Reals$, generated by the process $\mu$, and an action $a_t \in
  \CA$, chosen by the decision maker. We denote the distribution over
  next states $s_{t+1}$, which only depends on $s_t$ and $a_t$, by
  $\mu(s_{t+1}|s_t,a_t)$.  Furthermore $\mu(r_{t+1}|s_t, a_t)$ is a
  reward distribution conditioned on states and actions.  Finally,
  $\mu(\rtn, \stn | \st, \at) = \mu(\rtn | \st, \at)\mu(\stn | \st,
  \at)$.  \label{def:mdp}
\end{definition}
In the above, and throughout the text, we usually take $\mu(\cdot)$ to mean
$\Pr_\mu(\cdot)$, the distribution under the process $\mu$, for
compactness. Frequently such a notation will imply a marginalisation.
For example, we shall write $\mu(s_{t+k}|s_t,a_t)$ to mean:
\[
\sum_{s_{t+1}, \ldots, s_{t+k-1}} \mu(s_{t+k}, \ldots,
s_{t+1}|s_t,a_t).
\]
The decision maker takes actions according to a policy $\pi$, which
defines a distribution $\pi(a_t | s_t)$ over $\CA$, conditioned on the
state $s_t$, i.e. a set of probability measures over $\CA$ indexed by
$s_t$.  A policy $\pi$ is stationary if $\pi(a_t=a|s_t=s) =
\pi(a_{t'}=a|s_{t'}=s)$ for all $t,t'$.  The expected utility of a
policy $\pi$ selecting actions in the MDP $\mu$, from time $t$ to $T$
can be written as the {\em value function}:
\begin{align}
  \label{eq:value_function}
  \Vtpm(s) = 
  \Ep{\pi,\mu}\left(
    \sum_{k=1}^T \gamma^k r_{t+k}
    \Big|
    s_t
  \right),
\end{align}
where $\Ep{\pi,\mu}$ denotes the expectation under the Markov chain
arising from acting policy $\pi$ on the MDP $\mu$.  Whenever it is
clear from context, superscripts and subscripts shall be omitted for
brevity. The {\em optimal} value function will be denoted by $V^*
\defn \max_\pi V^\pi$.  
If the MDP is known, we can evaluate the optimal value function policy
in time polynomial to the sizes of the state and action
sets~\citep{Puterman:MDP:1994} via backwards induction (value
iteration).

\subsection{Bayesian Reinforcement Learning}
If the MDP is unknown, we may use a Bayesian framework to represent
our uncertainty~\citep{duff2002olc}.  This requires maintaining a
belief $\xit$, about which MDP $\mu \in \MDPs$ corresponds to reality.
More precisely, we define a measurable space $(\MDPs, \FM)$, where
$\MDPs$ is a (usually uncountable) set of MDPs, and $\FM$ is a
suitable $\sigma$-algebra. With an appropriate initial density
$\xi_0(\mu)$, we can obtain a sequence of densities $\xi_t(\mu)$,
representing our subjective belief at time $t$, by conditioning
$\xi_t(\mu)$ on the latest observations:
\begin{align}
  \xitn(\mu)
  &\defn
  \frac{\mu(\rtn, \stn | \st, \at) \xit(\mu)}
  {\int_{\MDPs}\mu'(\rtn,\stn | \st, \at) \, \xi_t(\mu') \dd\mu'}.
\end{align}
In the following, we write $\Exi$ to denote expectations with respect
to any belief $\xi$.  

\subsection{Belief-Augmented MDPs}

In order to optimally select actions in this framework, it is
necessary to {\em explicitly} take into account future changes in the
belief when planning~\citep{duff2002olc}.  The idea is to combine the
original MDP's state $\st$ and our belief state $\xit$ into a {\em
  hyper-state}. 
\begin{definition}[BAMDP]
  A {\em Belief-Augmented MDP} $\nu$ (BAMPD) is an MDP with a set of
  hyper-states $\Omega = \CS \times \Beliefs$, where $\Beliefs$ is an
  appropriate set of probability measures on $\MDPs$ and $\CS, \CA$
  are the state and action sets of all $\mu \in \MDPs$.  At time $t$,
  the agent observes the hyper-state $\wt = (\st, \xit) \in \Omega$
  and takes action $a_t \in \CA$.  We write the transition
  distribution as $\nu(\wtn | \wt, \at)$ and the reward distribution
  as $\nu(\rt | \wt)$.
\end{definition}
The hyper-state $\wt$ has the Markov property.  
This allows us to treat the BAMDP as an infinite-state MDP with
transitions $\nu(\omega_{t+1} | \omega_t, a_t)$, and rewards
$\nu(r_t|\omega_t)$.\footnote{Because of the way that the BAMDP $\nu$
  is constructed from beliefs over $\MDPs$, the next reward now
  depends on the next state rather than the current state and action.}
When the horizon $T$ is finite, we need only require expand the tree
 to depth $T - t$. Thus, backwards induction
 starting from the set of terminal hyper-states $\Omega_T$ and
 proceeding backwards to $T-1, \ldots, t$ provides a solution:
\begin{equation}
  V^*_n(\omega) = \max_{a \in \CA} \Ep{\nu}(r|\omega) + \gamma \!\!\!\!\!\! \sum_{\omega' \in \Omega_{n+1}} \!\!\!\! \nu(\omega'|\omega,a) V_{n+1}^* (\omega'),
  \label{eq:backwards_induction}
\end{equation}
where $\Omega_n$ is the set of hyper-states at time $n$.  We can
approximately solve infinite-horizon problems if we expand the tree to
some finite depth, if we have bounds on the value of leaf nodes.

\subsection{Bounds on the Value Function}
\label{sec:value_function_bounds}
We shall relate the optimal value function of the BAMDP,
$V^*(\omega)$, for some $\omega(s,\xi)$, to the value functions
$V_\mu^\pi$ of MDPs $\mu \in \MDPs$ for some $\pi$.  The optimal
policy for $\mu$ is denoted as $\pi^*(\mu)$. The {\em mean} MDP
resulting from belief $\xi$ is denoted as $\meanMDP$ and has the
properties: $\meanMDP(\stn|\st,\at) = \Exi[\mu(\stn|\st,\at)]$,
$\meanMDP(\rtn|\st,\at) = \Exi[\mu(\rtn|\st,\at)]$.
\begin{proposition}\cite{dimitrakakis:cimca08}
  For any $\omega = (s, \xi)$, the BAMDP value function $V^*$ obeys:
  \begin{equation}
    \int_{\MDPs} V_\mu^{\pi^*(\mu)}(s) \xi(\mu) d\mu 
    \geq 
    V^*(\omega)
    \geq
    \int_{\MDPs} V_\mu^{\pi^*(\meanMDP)}(s) \xi(\mu) \, d\mu
    \label{eq:value_bounds}
  \end{equation}
  \label{prop:value_bounds}
\end{proposition}
\begin{proof}
  By definition, $V^*(\omega) \geq V^\pi(\omega)$ for all $\omega$,
  for any policy $\pi$. It is easy to see that the lower bound equals
  $V^{\pi^*(\meanMDP)}(\omega)$, thus proving the right hand side.
  The upper bound follows from the fact that for any function $f$,
  $\max_x \int f(x, u) \,du \leq \int \max_x f(x, u) \,du$.
\lqed
\end{proof}


If $\CM$ is not finite, then we cannot calculate the upper bound of
$V(\omega)$ in closed form.  However, we can use Monte Carlo sampling:
Given a hyper-state $\omega = (s, \xi)$, we draw $m$ MDPs from its
belief $\xi$: $\mu_1, \ldots, \mu_m \sim \xi$,\footnote{In the
  discrete case, we sample a multinomial distribution from each of the
  Dirichlet densities independently for the transitions. For the rewards we
  draw independent Bernoulli distributions from the Beta of each
  state-action pair.}  estimate the value function for
each $\mu_k$, $\tilde{v}^\omega_{U,k} \defn V^{\pi^*(\mu_k)}_{\mu_k}
(s)$, and average the samples:
$\hat{v}_{U,m}^\omega \defn \frac{1}{m}\sum_{k=1}^m \tilde{v}^\omega_{U,k}$.
Let $v^\omega_{U} \defn \int_{\CM} \xi_\omega(\mu) V_\mu^*(s_\omega)
\, d\mu$. Then, $\lim_{m \to \infty} [\hat{v}^\omega_{U,m}] =
v^\omega_U$ almost surely and $\E[\hat{v}^\omega_{U,m}] =
v^{\omega}$.

Lower bounds can be calculated via a similar procedure.  We begin by
calculating the optimal policy $\pi^*(\meanMDP)$ for the mean MDP
$\meanMDP$ arising from $\xi$. We then compute
$\tilde{v}^\omega_{L,k} \defn V^{\pi^*(\meanMDP)}_{\mu_k}$, the value
of that policy for each sample $\mu_k$ and estimate
$\hat{v}_{L,m}^\omega \defn \frac{1}{m}\sum_{k=1}^m
\tilde{v}^\omega_{L,k}$.

\section{\SectStyle{Complexity of belief tree search}}
\label{sec:pac-bounds}
We now present our main results.  Detailed proofs are given in the
appendix.  We search trees which arise in the context of planning
under uncertainty in MDPs using the BAMDP framework. We can use value
function bounds on the leaf nodes of a partially expanded BAMDP tree
to obtain bounds for the inner nodes through backwards induction. The
bounds can be used both for action selection and for further tree
expansion.  However, the bounds are estimated via Monte Carlo
sampling, something that necessitates the use of stochastic branch and
bound technique to expand the tree.

We analyse a set of such algorithms. The first is a search to a fixed
depth that employs exact lower bounds. We then show that if only
stochastic bounds are available, the complexity of fixed depth search
only increases logarithmically.  We then present two stochastic branch
and bound algorithms, whose complexity is dependent on the number of
near-optimal branches. The first of these uses bound samples on leaf
nodes only, while the second uses samples obtained in the last half of
the parents of leaf nodes, thus using the collected samples more
efficiently.
\begin{figure}
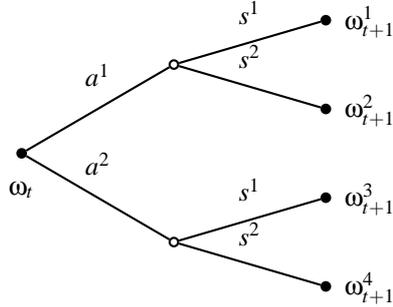

  \begin{center}
    $
    \pstree[treemode=R,radius=2pt]
    {\TC*[radius=2pt]~{\omega_t}}
    {
      \pstree{\Tc{2pt}\taput{a^1}}
      {
        \TC*~{\omega_{t+1}^1}\taput{s^1}
        \TC*~{\omega_{t+1}^2}\taput{s^2}
      }
      \pstree{\Tc{2pt}\taput{a^2}}
      {
        \TC*~{\omega_{t+1}^3}\taput{s^1}
        \TC*~{\omega_{t+1}^4}\taput{s^2}
      }
    }
    $
    \caption{A belief tree, where the rewards are ignored for
      simplicity, with actions $\CA=\{a^1, a^2\}$ and states
      $\CS=\{s^1, s^2\}$.}
    \label{fig:belief-tree}
  \end{center}
\end{figure}

\subsection{Assumptions and Notation}
We present the main assumptions concerning the tree search, pointing
out the relations to Bayesian RL.  The symbols $V$ and $v$ have been
overloaded to make this correspondence more apparent.  The tree that
has a branching factor at most $\phi$. The branching is due to both
action choices and random outcomes (see Fig.\ref{fig:belief-tree}).
Thus, the nodes at depth $k$ correspond to the set of hyper-states
$\{\omega_{t+k}\}$ in the BAMDP.  By abusing notation, we may also
refer to the components of each node $\omega = (s, \xi)$ as
$s(\omega), \xi(\omega)$.

We define a branch $b$ as a {\em set} of {\em policies} (i.e.  the set
of all policies starting with a particular action). The value of a
branch $b$ is $V^b \defn \max_{\pi \in b} V^\pi$. The root branch is
the set of all policies, with value $V^*$. A hyper-state $\omega$ is
$b$-reachable if $\exists \pi \in b$ s.t
$\Pr_{\pi,\nu}(\omega|\omega_t) > 0$.Any branch $b$ can be partitioned
at any $b$-reachable $\omega$ into a set of branches $B(b,\omega)$. A
possible partition is any $b_i = \cset{\pi \in b}{i=\argmax_a
  \pi(a|\omega)}$ for any $b_i \in B(b,\omega)$.  We simplify this by
considering only deterministic policies.
We denote the $k$-horizon value function by $V^b(k) \defn \max_{\pi
  \in b} V^\pi_{t,k}(\omega_t)$. For each tree node $\omega = (s,
\xi)$, we define upper and lower bounds $v_U(\omega) \defn
\Ep{\xi}[V^*_\mu(s)]$, $v_L(\omega) \defn \Ep{\xi}[V^{\pi^*(\meanMDP)}(s)]$, from
\eqref{eq:value_bounds}.  By fully expanding the tree to depth $k$ and
performing backwards induction \eqref{eq:backwards_induction}, using
either $v_U$ or $v_L$ as the value of leaf nodes, we obtain
respectively upper and lower bounds $\VU^b(k), \VL^b(k)$ on the value
of any branch.  Finally, we use $\CC(\omega)$ for the set of immediate
children of a node $\omega$ and the short-hand $\Omega_k$ for
$\CC^k(\omega)$, the set of all children of $\omega$ at depth $k$. We
assume the following:
\begin{assumption}[Uniform linear convergence]
  There exists $\gamma \in (0,1)$ and $\beta > 0 $ s.t. for any
  branch $b$, and depth $k$, $\Vb - \VLb(k) \leq \beta \gamma^k$,
  $\VUb(k) - \Vb \leq \beta \gamma^k$.
  \label{ass:convergence}
\end{assumption}
\begin{remark}
  For BAMDPs with $r_t \in [0,1]$ and $\gamma < 1$,
  Ass.~\ref{ass:convergence} holds, from boundedness and the
  geometric series,
  with $\beta =
  1/(1-\gamma)$, since $\VL^b(k)$ and $\VU^b(k)$ are the $k$-horizon
  value functions with the value of leaf nodes bounded in $1/(1-\gamma)$.
\end{remark}

We analyse algorithms which search the tree and then select an (action)
branch $\hat{b}^*$. For each algorithm, we examine the number of leaf
node evaluations required to bound the regret $V^* - V^{\hat{b}^*}$.

\subsection{Flat Search}
With exact bounds, we can expand all branches to a fixed depth and
then select the branch $\hat{b}^*$, with the highest lower bound. This
is Alg.~\ref{alg:flat_oracle_search}, with complexity given by the
following lemma.
\begin{algorithm}[t]
  \caption{Flat oracle search}
  \label{alg:flat_oracle_search}
  \begin{algorithmic}[1]
    \STATE Expand all branches until depth $k = \log_\gamma \epsilon/\beta$
    or $\hDL > \beta\gamma^k - \epsilon$.
    \STATE Select the root branch $\hbs = \argmax_b \VLb(k)$.
  \end{algorithmic}
\end{algorithm}
\begin{lemma}
  Alg.~\ref{alg:flat_oracle_search} on a tree with branching factor
  $\phi$, $\gamma \in (0,1)$, samples $\CO(\phi^{1+\log_\gamma
      \epsilon/\beta})$ times to bound the regret by
  $\epsilon$.
  \label{lem:flat_oracle}
\end{lemma}
\begin{proof}Bound the $k$-horizon value function error with
  Ass.~\ref{ass:convergence} and note that there are $\phi^{k+1}$
  leaves.
\end{proof}
\iftrue
In our case, we only have a stochastic lower bound on the value of
each node.  Algorithm~\ref{alg:flat_stochastic_search} expands the
tree to a fixed depth and then takes multiple samples from each leaf
node.
\begin{algorithm}[t]
  \caption{Flat stochastic search}
  \label{alg:flat_stochastic_search}
  \begin{algorithmic}[1]
    \STATE {\sc FSSearch}$(\omega_t,k,m)$
    \STATE Let $\Omega_k = \cset{\omega_{t+k}^i}{i=1,\ldots,\phi^k}$ 
    be the set of all $k$-step children of $\omega$
    \FOR {$\omega \in \Omega_k$}
    \STATE Draw $m$ samples $\tilde{v}^\omega_{L,j} = V_\mu^\pi$,
    $\mu \sim \xi(\omega)$
    \STATE $\hat{v}_L^\omega = \frac{1}{m} \sum_{j=1}^m \tilde{v}^\omega_{L,j}$,
    \ENDFOR
    \STATE Calculate $\hat{V}^b$
    \RETURN $\hat{b}^* = \argmax \hat{V}^b$.
  \end{algorithmic}
\end{algorithm}
\fi
\iftrue
\begin{lemma}
  Calling Alg.~\ref{alg:flat_stochastic_search} with $k =
  \lceil\log_\gamma \epsilon/2\beta\rceil$, $m = 2 \lceil \log_\gamma
  (\epsilon/2\beta) \rceil \cdot \log \phi$, we bound the regret by
  $\epsilon$ using
  $\CO\left(
    \phi^{1 + \log_\gamma \epsilon/2\beta} \log_\gamma (\epsilon/2\beta) \cdot \log \phi
  \right)$
  samples.
  \label{lem:flat_stochastic_search}
\end{lemma}
\begin{proof}The regret now is due to both limited depth and
  stochasticity. We bound each by $\epsilon/2$, the first via
  Lem.~\ref{lem:flat_oracle} and the second via Hoeffding's
  inequality.
\end{proof}
Thus, stochasticity mainly adds a logarithmic factor to the oracle
search. We now consider two algorithms which do not search to a fixed
depth, but select branches to deepen adaptively.
\fi


\subsection{Stochastic Branch and Bound 1}
A stochastic branch and bound algorithm similar to those examined here
was originally developed by \citet{Norkin:StochasticBnB:MP96} for
optimisation problems. At each stage, it takes an additional sample at
each leaf node, to improve their upper bound estimates, then
expands the node with the highest mean upper
bound. Algorithm~\ref{alg:stochastic-branch-and-bound-1} uses the same
basic idea, averaging the value function samples at every leaf node.

\begin{algorithm}[t]
  \caption{Stochastic branch and bound 1}
  \label{alg:stochastic-branch-and-bound-1}
  \begin{algorithmic}[1]
    \STATE Let $\CL_0$ be the root.
    \FOR {$n=1,2,\ldots$}
    \FOR {$\omega \in \CL_n$}
    \STATE $m_\omega{++}$,
    $\mu \sim \xi(\omega)$, 
    $\tv_{m_{\omega}}^\omega = V^*_\mu(s(\omega))$.
    \STATE $\hat{v}_U^\omega = \frac{1}{m_\omega}\sum_{i=1}^{m_\omega} \tv_i^\omega$
    \ENDFOR
    \STATE $\hat{\omega}^*_n = \argmax_\omega \hat{v}_U^\omega$.
    \STATE $\CL_{n+1} = \CC(\hat{\omega}^*_n) \cup \CL_n \backslash \hat{\omega}^*_n$
    \ENDFOR
  \end{algorithmic}
\end{algorithm}

In order to bound complexity, we need to bound the time required until
we discover a nearly optimal branch. We calculate the number of times
a suboptimal branch is expanded before its suboptimality is
discovered.  Similarly, we calculate the number of times we
shall sample the optimal node until its mean upper bound becomes
dominant.  These two results cover the time spent sampling upper
bounds of nodes in the optimal branch without expanding them and the
time spent expanding nodes in a sub-optimal branch.

\begin{lemma}
  If $N$ is the (random) number of samples $\tilde{v}_i$ from random
  variable $V \in [0,\beta]$ we must take until its empirical mean
  $\hat{V}_k \defn \sum_{i=1}^k \tilde{v}_i > \E V - \Delta$, then:
  \begin{align}
    \E[N] &\leq 1 + \beta^2\Delta^{-2}
    \label{eq:expected_leaf_samples}
    \\
    \Pr[N > n] &\leq \exp\left(-2\beta^{-2}n^2\Delta^2\right).
    \label{eq:probability_leaf_samples}
  \end{align}
  \label{lem:leaf_samples}
\end{lemma}
\begin{proof}The first inequality follows from the Hoeffding
inequality and an integral bound on the resulting sum, while the
second inequality is proven directly via a Hoeffding bound.
\end{proof}
By setting $\Delta$ to be the difference between the optimal and
second optimal branch, we can use the above lemma to bound the number
of times $N$ the leaf nodes in the optimal branch will be sampled
without being expanded.  The converse problem is bounding the number
of times that a suboptimal branch will be expanded.
\begin{lemma}
  If $b$ is a branch with $V^b = V^* - \Delta$, then it will be
  expanded at least to depth $k_0 = \log_\gamma \Delta/\beta$.
  Subsequently,
        \begin{equation}
      \Pr(K > k)
      <
      \CO\left(\exp\left\{
      -2\beta^{-2}
      \left[
      (k - k_0)\Delta^2
      \right]
      \right\}
      \right).
    \end{equation}
      \label{lem:suboptimal-branch-1}
\end{lemma}
\begin{proof} In the worst case, the branch is degenerate and only one
  leaf has non-zero probability. We then apply a Hoeffding bound to
  obtain the desired result.
\end{proof}

\subsection{Stochastic Branch and Bound 2}
\begin{algorithm}[t]
  \caption{Stochastic branch and bound 2}
  \label{alg:stochastic-branch-and-bound-2}
  \begin{algorithmic}[1]
    \FOR {$\omega \in \CL_{n}$}
    \STATE $\hat{V}_U^\omega = \frac{1}{\sum_{\omega' \in \CC(\omega)} m_{\omega'}} \sum_{\omega' \in \CC(\omega)} \sum_{i=1}^{m_\omega'} \tv_i^{\omega'}$
    \ENDFOR
    \STATE Use \eqref{eq:backwards_induction} to obtain $\hVU$ for all nodes.
    \STATE Set $\omega_0$ to root.
    \FOR {$d=1,\ldots$}
    \STATE $a^*_d = \argmax_a \sum_{\omega \in \Omega_d} \omega_{d-1}(j|a) \hVU(\omega)$
    \STATE $\omega_d \sim \omega_{d-1}(j|a^*_d)$
    \IF{$\omega_d \in \CL_n$}
    \STATE $\CL_{n+1} = \CC(\omega_d) \cup \CL_n \backslash \omega_d$
    \STATE {\bf Break}
    \ENDIF
    \ENDFOR
  \end{algorithmic}
\end{algorithm}
The degeneracy is the main problem of
Alg.~\ref{alg:stochastic-branch-and-bound-1}.
Alg.~\ref{alg:stochastic-branch-and-bound-2} not only propagates upper
bounds from multiple leaf nodes to the root, but also re-uses upper
bound samples from inner nodes, in order to handle the degenerate case
where only one path has non-zero probability. (Nevertheless,
Lemma~\ref{lem:leaf_samples} applies without modification to
Alg.~\ref{alg:stochastic-branch-and-bound-2}). Because we are no
longer operating on leaf nodes, we can take advantage of the upper
bound samples collected along a given trajectory. However, if we use
all of the upper bounds along a branch, then the early samples may
bias our estimates a lot.  For this reason, if a leaf is at depth $k$,
we only average the upper bounds along the branch to depth $k/2$.
The complexity of this approach is given by the following lemma:
\begin{lemma}
  If $b$ is s.t. $V^b = V^* - \Delta$, it will be expanded to depth
  $k_0 > \log_\gamma \Delta/\beta$ and
  \begin{align*}
    \Pr(K > k)
    &\lessapprox
    \exp\left(
      -2(k-k_0)^2(1-\gamma^2)
    \right),
    & k &> k_0
  \end{align*}
  \label{lem:suboptimal-branch-2}
\end{lemma}
\begin{proof} There is a degenerate case where only one sub-branch has
  non-zero probability.  However we now re-use the samples that were
  obtained at previous expansions, thus allowing us to upper bound the
  bias by $\frac{\Delta (1-\gamma^{k+1})}{(k-k_0)(1-\gamma)}$. This
  allows to use a tighter Hoeffding bound and so obtain the desired
  outcome.
\end{proof}
This bound decreases faster with $k$. Furthermore, there is no
dependence on $\Delta$ after the initial transitory period, which may
however be very long.  The gain is due to the fact that we are
re-using the upper bounds previously obtained in inner nodes.  Thus,
this algorithm should be particularly suitable for stochastic problems.


\subsection{Lower Bounds for Bayesian RL}

We can reduce the branching factor $\phi$, (which is $|\CA \times \CS
\times \CR|$ for a full search) by employing sparse sampling
methods~\citep{DBLP:conf/ijcai/KearnsMN99} to
$\CO\{|\CA|\exp[1/(1-\gamma)]\}$. This was essentially the approach
employed by~\citep{wang:bayesian-sparse-sampling:icml:2005}.  However,
our main focus here is to reduce the depth to which each branch is
searched.  

The main problem with the above algorithms is the fact that we must
reach $k_0 = \lceil \log_\gamma \Delta \rceil$ to discard
$\Delta$-optimal branches. 
However, since the hyper-state $\omega_t$ arises from a Bayesian
belief, we can use an additional smoothness property:
\begin{lemma}
  The Dirichlet parameter sequence $\psi_{t}/n_t$, with $n_t \defn
  \sum_{i=1}^K\psi^i_t$, is a $c$-Lipschitz martingale with $c_t =
  1/2(n_t+1)$.
  \label{lem:dirichlet-martingale}
\end{lemma}
\begin{proof} Simple calculations show that, no matter what is
  observed, $\E_{\xi_t}(\psi_{t+1}/n_{t+1}) = \psi_t / n_t$.  Then, we
  bound the difference $|\psi_{t+k} / n_{t+k} - \psi_t / n_t|$ by two
  different bounds, which we equate to obtain $c_t$.
\end{proof}
\begin{lemma}
  If $\mu$, $\hat{\mu}$ are such that $\|\CT - \hat{\CT}\|_\infty \leq
  \epsilon$ and $\|r - \hat{r}\|_\infty \leq \epsilon$, for some
  $\epsilon > 0$, then
  $\inftynorm{V^\pi - \hat{V}^\pi} \leq \frac{\epsilon}{(1-\gamma)^2}$,
  for any policy $\pi$.
  \label{lem:value_function_error}
\end{lemma}
\begin{proof}
  By subtracting the Bellman equations for $V, \hat{V}$ and taking the
  norm, we can repeatedly apply Cauchy-Schwarz and triangle
  inequalities to obtain the desired result.
\end{proof}
The above results help us obtain better lower bounds in two
ways. First we note that initially $1/k$ converges faster than
$\gamma^k$, for large $\gamma$, thus we should be able to expand less
deeply.  Later, $n_t$ is large so we can sample even more sparely.

If we search to depth $k$, and the rewards are in $[0,1]$, then,
naively, our error is bounded by $\sum_{n=k}^\infty \gamma^n =
\gamma^k / (1-\gamma)$.  However, the mean MDPs for $n > k$ are close
to the mean MDP at $k$ due to Lem.~\ref{lem:dirichlet-martingale}.
This means that $\beta$ can be significantly smaller than
$1/(1-\gamma)$.  In fact, the total error is bounded by
$\sum_{n=k}^\infty \gamma^n (n-k)/n$.  For undiscounted problems, our
error is bounded by $T-k$ in the original case and by
$T-k[1+\log(T/k)]$ when taking into account the smoothness.

\section{\SectStyle{Conclusions and related work}}
\label{sec:related_work}
Much recent work on Bayesian RL focused on myopic estimates or full
expansion of the belief tree up to a certain depth.  Exceptions
include~\citep{poupart2006asd}, which uses an analytical bound based
on sampling a small set of beliefs
and~\citep{wang:bayesian-sparse-sampling:icml:2005}, which uses
Kearn's sparse sampling algorithm~\citep{DBLP:conf/ijcai/KearnsMN99}
to expand the tree.  Both methods have complexity exponential in the
horizon, something which we improve via the use of smoothness
properties induced by the Bayesian updating.

There are also connections with work on POMDPs
problems~\citep{RossPineau:OnlinePlanningPOMDPs:jmlr2008}. However
this setting, though equivalent in an abstract sense, is not
sufficiently close to the one we consider.  Results on bandit
problems, employing the same value function bounds used herein were
reported in~\citep{dimitrakakis:cimca08}, which experimentally
compared algorithms operating on leaf nodes only.

Related results on the online sample complexity of Bayesian RL were
developed by \citep{Kolter-Ng:NearBayesianExploration}, who employs a
different upper bound to ours and \citep{Asmuth:BOSS}, who employs MDP
samples to plan in an augmented MDP space, similarly to
\cite{DBLP:conf/nips/AuerJO08} (who consider the set of plausible
MDPs) and uses Bayesian concentration of measure
results~\citep{Zhang:EntropyDensityEstimation} to prove mistake bounds
on the online performance of the algorithm.

Interestingly, Alg.~\ref{alg:stochastic-branch-and-bound-2} resembles
HOO~\citep{DBLP:conf/nips/BubeckMSS08} in the way that it traverses
the tree, with two major differences.
\begin{inparaenum}[(a)]
\item The search is adapted to {\em stochastic} trees.  
\item We use means of samples of upper bounds, rather than upper
  bounds on sample means.
\end{inparaenum}
For these reasons, we are unable to simply restate the arguments
in~\citep{DBLP:conf/nips/BubeckMSS08}.  

We presented complexity results and counting arguments for a number of
tree search algorithms on trees where stochastic upper and lower
bounds satisfying a smoothness property exist. These are the first
results of this type and partially extend the results
of~\citep{Norkin:StochasticBnB:MP96}, which provided an asymptotic
convergence proof, under similar smoothness conditions, for a
stochastic branch and bound algorithm.  In addition, we introduce a
mechanism to utilise samples obtained at inner nodes when calculating
mean upper bounds at leaf nodes.  Finally, we relate our complexity
results to those of~\citep{DBLP:conf/ijcai/KearnsMN99}, for whose
lower bound we provide a small improvement.  We plan to address the
online sample complexity of the proposed algorithms, as well as their
practical performance, in future work.

\section*{\SectStyle{Acknowledgements}}
This work was part of the ICIS project, supported by the Dutch
Ministry of Economic Affairs, grant nr: BSIK03024. I would also like
to thank the anonymous reviewers for their detailed reviews of earlier
versions of this paper; Peter Auer, Peter Gr\"{u}nwald, Ronald Ortner
and Remi Munos for extensive discussions; and finally Shimon Whiteson
and Frans Groen for further comments and corrections.

\appendix
\section{Proofs of the main results}
\label{sec:proofs}
\iftrue
\begin{proof}[Proposition~\ref{prop:value_bounds}]
  By definition, $V^*(\omega) \geq V^\pi(\omega)$ for all $\omega$,
  for any policy $\pi$.  The lower bound follows trivially, since
  \begin{equation}
    \label{eq:value_lower_bound}
    V^{\pi^*(\hat{\mu}_\omega)}(\omega) \defn \int
    V_\mu^{\pi^*(\hat{\mu}_\omega)}(s_\omega) \xi_\omega(\mu) \, d\mu.
  \end{equation}
  The upper bound is derived as follows.  First note that for any
  function $f$, $\max_x \int f(x, u) \,du \leq \int \max_x f(x, u)
  \,du$.  Then, we remark that:
  \begin{subequations}
    \begin{align}
      V^*(\omega)
      &=
      \max_\pi \int V_\mu^{\pi}(s_\omega) \xi_\omega(\mu) \, d\mu
      \label{eq:value_star_omega}
      \\
      &\leq
      \int \max_\pi V_\mu^{\pi}(s_\omega) \xi_\omega(\mu) \, d\mu 
      \label{eq:expected_max_value}
      \\
      &=
      \int V_\mu^{\pi^*(\mu)}(s_\omega) \xi_\omega(\mu) \, d\mu.
      \label{eq:expected_max_policy_value}
    \end{align}
  \end{subequations}
\lqed
\end{proof}
\fi
\begin{proof}[Lemma~\ref{lem:flat_oracle}]
  For any $b'$ with $\VLbp < \VLb$, we have: $\Vbp \leq \VLbp +
  \beta\gamma^k < \VLb + \beta\gamma^k \leq \Vb + \beta\gamma^k$.
  This holds for $b = \hbs$.  Thus, in the worst case, the regret that
  we suffer if there exists some $b' : \Vbp > \Vhbs$ is
  $\epsilon = \Vbp - \Vhbs < \beta \gamma^k$.
  To reach depth $k$ in all branches we need $n = \sum_{t=1}^{k}
  \phi^k < \phi^{k+1}$ expansions.  Thus, we require $k=
  \frac{\log(\epsilon/\beta)}{\log \gamma}$ and $n \leq
  \phi^{1 + \log_\gamma(\epsilon/\beta)}$.
\lqed
\end{proof}
\iftrue
\begin{proof}[Lemma~\ref{lem:flat_stochastic_search}]
  The total number of samples is $km$, the number of leaf nodes times
  the number of samples at each leaf node.  The search is until depth 
  \begin{align}
    k
    &= \Lceil\log_\gamma \epsilon/2\beta\Rceil
    \leq 1 + \log_\gamma \epsilon/2\beta
  \end{align}
  and the number of samples is 
  \begin{align}
    m = 2 \log_\gamma(\epsilon/2\beta) \log\phi.
  \end{align}
  The complexity follows trivially.  Now we must prove that this
  bounds the expected regret with $\epsilon$.  Note that
  $\beta\gamma^k < \epsilon/2$, so for all branches $b$:
  \begin{align}
    \hVLb - \Vb &< \epsilon/2.
  \end{align}
  The expected regret can now be written as
  \begin{align}
    \E R &\leq 
    \frac{\epsilon}{2}
    + \E[R | \hat{V}_L^{\hbs} < \hat{V}_L^{b^*} + \epsilon/4]
    \Pr(\hat{V}_L^{\hbs} < \hat{V}_L^{b^*} + \epsilon/4)
    \\
    &+ \E[R | \hat{V}_L^{\hbs} \geq \hat{V}_L^{b^*} + \epsilon/4]
    \Pr(\hat{V}_L^{\hbs} \geq \hat{V}_L^{b^*} + \epsilon/4).
  \end{align}
  From the Hoeffding bound \eqref{eq:hoeffding}
  \[
  \Pr(\hVL - \VL > \epsilon/4) 
  <
  \exp\left(-\frac{1}{8}m\beta^{-2}\gamma^{-2k}\epsilon^2\right)
  \]
  and with a union bound the total error probability is bounded by
  $\phi^k
  \exp\left(-\frac{1}{8}m\beta^{-2}\gamma^{-2k}\epsilon^2\right)$.  If
  our estimates are within $\epsilon/4$ then the sample regret is
  bounded by $\epsilon/4$, while the other terms are trivially bounded
  by 1, to obtain
  \begin{equation}
    \E R \leq 
    \frac{\epsilon}{2}
    + \left\{
      \phi^k \exp\left(-\frac{1}{8}m\beta^{-2}\gamma^{-2k}\epsilon^2\right)+
      \frac{\epsilon}{4}.
    \right\}
  \end{equation}
  Substituting $m$ and $k$, we obtain the stated result.
\lqed
\end{proof}
\fi

\begin{proof}[Lemma~\ref{lem:leaf_samples}]
  \begin{align}
    \E[N]
    &= 
    \sum_{n=1}^\infty n \prod_{j=1}^{n-1} \Pr(\hV(j) \geq V + \epsilon)
    \Pr(\hV(n) < V + \epsilon)
    \\
    &\leq
    \sum_{n=1}^\infty n \exp\left(-2\beta^{-2}\epsilon^2\sum_{j=1}^{n-1}j\right)
    =
    \sum_{n=1}^\infty n \exp\left(-\beta^{-2}\epsilon^2n(n+1)\right)
  \end{align}
  Let us now set $\rho = \exp(-\beta^{-2}\epsilon^2)$.  Observe that
  $n\rho^{n(n+1)} < n\rho^{n^2}$, since $\rho < 1$. Then, note that
  $\int n\rho^{n^2} \, dn = \CO\left(\frac{\rho^{n^2}}{2\log
      \rho}\right)$.  So we can bound the sum by
  \begin{align}
    \sum_{n=1}^\infty n \rho^{n(n+1)}
    &< 1 + \left[\frac{\rho^{n^2}}{2\log \rho}\right]_1^\infty
    1 + \frac{\exp(-\beta^{-2}\epsilon^2)}{2\beta^{-2}\epsilon^2} <  1 + \left(\frac{\beta}{\epsilon}\right)^2.
  \end{align}
  This proves the first inequality.  For the second inequality, we have:
  \begin{align}
    \Pr(N > n) &= \Pr\left(\bigwedge_{k=1}^n \hat{V}(k) > V + \epsilon\right)
    < \prod_{k=1}^n \exp\left(-2k\beta^{-2}\epsilon^2\right)
    \\
    &= \exp\left(-\beta^{-2}\epsilon^2n(n+1)\right)
    < \exp\left(-n^2\beta^{-2}\epsilon^2\right).
  \end{align}
  This completes the proof for the first case.  The second case ise
  symmetric.
\lqed
\end{proof}

\begin{proof}[Lemma~\ref{lem:suboptimal-branch-1}]
  In order to stop expanding a sub-optimal branch $b$,
  at depth $k$, we must have $V_U^b(k) < V^*$, since in the worst case
  $V_U^*(k) = V^*$ for all $k$. Since $V^b = V^* - \Delta$, this only
  happens when $k$ is greater than
  $k_0 \defn \Lceil \log_\gamma \Delta/\beta \Rceil$,
  which is the minimum depth we must expand to.  Subsequently, we
  shall note that the probability of stopping is $\Pr(\hVUb(k) >
  \Delta - \beta\gamma^k) < \exp(-2(\Delta -
  \beta\gamma^k)^2\beta^{-2})$.  We can not do better due to the
  degenerate case where only one leaf node of the branch has non-zero
  probability.

  The probability of not stopping at depth $k$ is bounded by:
  \begin{align*}
    \Pr(K > k)
    &\leq
    \prod_{j=k_0}^k \exp(-2(\Delta - \beta\gamma^j)^2\beta^{-2})
    \leq
    \exp\left(-2\beta^{-2}\sum_{j=k_0}^k(\Delta - \beta\gamma^j)^2\right)
    \\
    &\leq
    \exp\left[ -\frac{2}{\beta^2}\left((k-k_0)\Delta^2 + \frac{\beta h}{1 - \gamma^2} \right) \right],
    \\
    h&= \beta(\gamma^{2k_0} - \gamma^{2(k+1)} - 2\Delta(\gamma^{k_0} - \gamma^{k+1})(1+\gamma)
    \\
    &= \beta(\Delta^2 - \gamma^{2(k+1)} - 2\Delta(\Delta - \gamma^{k+1})(1+\gamma).
  \end{align*}
\lqed
\end{proof}

\begin{proof}[Lemma~\ref{lem:suboptimal-branch-2}]
  Similarly to the previous lemma, there is a degenerate case where
  only one sub-branch has non-zero probability.  However this
  algorithm re-uses the samples that were obtained at previous
  expansions.
  When at depth $k$, we average the bounds from $\Lceil k/2
  \Rceil$ to $k$. Since, in the worst case, we cannot stop until $k >
  k_0 = \Lceil \log_\gamma \Delta / \beta \Rceil$, we shall bound the
  probability that we stop at some depth $K > 2k_0$.  Then the mean
  upper bound bias is at most:
  \[
  h_k \defn \frac{1}{k-k_0} \sum_{n=k_0}^k \beta \gamma^n
  =
  \frac{\beta \gamma^{k_0}}{k-k_0} \frac{1-\gamma^{k+1}}{1-\gamma}
  <
  \frac{\Delta}{k-k_0} \frac{1-\gamma^{k+1}}{1-\gamma}.
  \]
  The procedure continues only if the sampling error exceeds $\Delta
  - h_k$, so it suffices to bound $\Pr(\hat{X}_k > \bar{X}_k +
  \epsilon)$, where $\hat{X}_{k} = \sum_{n=\Lceil k/2 \Rceil}^k
  \hVU(k)$ and $\bar{X}_{k} = V + h_k$ for $\epsilon = \Delta(1 -
  \frac{1-\gamma^k}{(k-k_0)(1-\gamma)})$:
  $\Pr(\hat{X}_k > \bar{X}_k + \epsilon) <
  \exp\left(-\frac{2(k-k_0)^2\epsilon^2}{\sum_{n=k_0}^k
      (\beta\gamma^{n})^2}\right)$.
  Since $\sum_{n=k_0}^k (\beta \gamma^n)^2 = 
  \Delta^2 \frac{1-\gamma^{2(k+1)}}{1-\gamma^2}$:
  $\Pr(\hat{X}_k > \bar{X}_k + \epsilon) <
  \exp\left(
    -\frac{2(k-k_0)^2(1-\gamma^2)\epsilon^2}
    {\Delta^2(1-\gamma^{2(k+1)})}
  \right)$.
  By setting $\epsilon = \Delta - h_k$ we can bound this  by
  \begin{align*}
    \exp\left(
      -\frac{2(k-k_0)^2(1-\gamma^2)}
      {(1-\gamma^{2(k+1)})}
      \cdot
      \left(1 - \frac{1-\gamma^{k+1}}{(k-k_0)(1-\gamma)}\right)^2
    \right).
  \end{align*}
  For large $k$, this is approximately $\CO(\exp(-k^2))$.  \lqed
\end{proof}

\iftrue
\begin{proof}[Lemma~\ref{lem:dirichlet-martingale}]
  It is easy to see that $\E(\psi_{t+1}/n_{t+1}|\xit) = \psi_t/n_t$.
  This follows trivially when no observations are made since
  $\psi_{t+1} = \psi_t$.  When one observation is made, $n_{t+1} = 1 +
  n_t$. Then $\E(\psi_{t+1}/n_{t+1}|\xit) = [\psi_t +
  \xit(\psi)]/n_{t+1} = [\psi_t + \psi_t/n_t]/(1+n_t) = \psi_t/n_t$.
  Thus, the matrix $\CT_{\xit}$ is a martingale.  We shall now prove
  the Lipschitz property. For all $k > 0$, $\psi_t > 0$:
  \[
  \psi^i_t/(n_t + k) \leq \psi^i_{t+k}/n_{t+k}
  \leq (\psi^i + k)/n_{t+k}.
  \]

  Note that $\left| \frac{\psi^i_{t+k}}{n_{t+k}} -
    \frac{\psi^i_{t}}{n_{t}} \right|$ is upper bounded by $\frac{k
    (n_t - \psi^i_t)}{n_t(n_{t}+k)}$ and $\frac{k
    \psi^i_t}{n_t(n_{t}+k)}$ and thus by $\frac{k \min\left\{\psi^i_t, n_t
      - \psi^i_k\right\} }{n_t(n_{t}+k)}$.
  Equating the two terms, we obtain
  $
  \left|
    \frac{\psi^i_{t+k}}{n_{t+k}}
    -
    \frac{\psi^i_{t}}{n_{t}}
  \right|
  \leq
  \frac{k}{2(n_{t}+k)}
  $.
  \lqed
\end{proof}
\fi

\begin{proof}[Lemma~\ref{lem:value_function_error}]
  The transitions $P, \hat{P}$ induced by any policy obey $\|P -
  \hat{P}\|_\infty < \epsilon$.  By repeated use of Cauchy-Schwarz and
  triangle inequalities:
  \begin{align*}
    \|V - \hat{V}\|_\infty
    &=
    \left\|
      r - \hat{r} + \gamma \left(P V - \hat{P}\hat{V}\right)
    \right\|_\infty
    \\
    &\leq
    \|r - \hat{r}\|_\infty
    + 
    \gamma
    \left\|
      P V - \hat{P}\hat{V}
    \right\|_\infty
    \\
    &\leq
    \epsilon
    + 
    \gamma
    \left\|
      P V - (P  - \tilde{P})\hat{V}
    \right\|_\infty
    \\
    &\leq
    \epsilon
    + 
    \gamma\left(
      \left\|
        P (V - \hat{V})
      \right\|_\infty
      +
      \left\|
        \tilde{P}\hat{V}
      \right\|_\infty
    \right)
    \\
    &\leq
    \epsilon
    + 
    \gamma\left(
      \|P \|_\infty\cdot\|V - \hat{V}\|_\infty
      +
      \|\tilde{P}\|_\infty\cdot\|\hat{V}\|_\infty
    \right)
    \\
    &\leq
    \epsilon
    + 
    \gamma\left(
      \left\|
        V - \hat{V}
      \right\|_\infty
      +
      \epsilon \cdot \frac{1}{1 - \gamma}
    \right)
  \end{align*}
  where $\tilde{P} = P - \hat{P}$, for which of course holds
  $\|\tilde{P}\|_\infty < \epsilon$.  Solving gives us the required
  result.
  \lqed
\end{proof}

\section{Hoeffding bounds for weighted averages}

Hoeffding bounds can also be derived for weighted averages.  Let us
first recall the standard Hoeffding inequality:
\begin{lemma}[Hoeffding inequality]
  If $\hx_n \defn \frac{1}{n} \sum_{i=1}^n x_i$, with $x_i \in
  [b_i,b_i+h_i]$ drawn from some arbitrary distribution $f_i$ and
  $\bar{x_n} \defn \frac{1}{n} \sum_i \E[x_i]$, then, for all $\epsilon \geq 0$:
  \begin{equation}
    \label{eq:hoeffding}
    \Pr\left(\hx_n \geq \bar{x}_n + \epsilon\right) \leq \exp\left(-\frac{2n^2 \epsilon^2}{\sum_{i=1}^n h_i^2}\right).
  \end{equation}
\end{lemma}
We have a weighted sum,
$\hx'_n \defn \sum_{i=1}^n w_i x'_i$, $\sum_{i=1}^n w_i = 1$.
If we set $v_i \defn n w_i$, then we can write the above as
$\frac{1}{n} \sum_{i=1}^n v_i x'_i$.  So, if we let $x_i = v_i x'_i$
and assume that $x'_i \in [b, b+h]$, then $x_i \in [v_i b + v_i (b +
h)]$.  Substituting into \eqref{eq:hoeffding} results in
\begin{align}
  \Pr\left(\hx_n \geq \bar{x} + \epsilon\right)
  \
  &\leq
  \exp\left(-\frac{2\epsilon^2}{h^2\sum_{i=1}^n w_i^2}\right).
\end{align}
Furthermore, note that
\begin{equation}
  \Pr\left(\hx_n \geq \bar{x} + \epsilon\right)
  <
  \exp\left(-\frac{2\epsilon^2}{h^2}\right),
  \label{eq:quick_hoeffding_bound}
\end{equation}
since $w_i^2 \leq w_i$ for all $i$, as $w_i \in [0,1]$.  Thus $\sum_i
w_i^2 \leq \sum_i w_i = 1$. Note that $\sum_i w_i^2 = 1$ iff $w_j = 1$
for some $j$.

\bibliographystyle{plainnat}
\bibliography{../../../bib/misc,../../../bib/mine}

\end{document}